\documentclass[a4paper]{article}

\usepackage[english]{babel}
\usepackage[utf8x]{inputenc}
\usepackage[T1]{fontenc}

\usepackage{times}
\usepackage{graphicx} 
\usepackage{subfigure} 

\usepackage{algorithm}
\usepackage{algorithmic}

\usepackage{amssymb}
\usepackage{amsmath}
\usepackage{amsthm}
\usepackage{tikz}
\usepackage{tikz-cd}
\usetikzlibrary{matrix, calc, arrows}
\usepackage{tabulary}
\usepackage{booktabs}
\usetikzlibrary{tikzmark}

\usepackage{array}

\newcolumntype{L}[1]{>{\raggedright\arraybackslash}p{#1}}
\newcolumntype{C}[1]{>{\centering\arraybackslash}p{#1}}
\newcolumntype{R}[1]{>{\raggedleft\arraybackslash}p{#1}}

\newcommand{\Id}{\textnormal{Id}}

\newcommand{\disc}{\textnormal{disc}}
\newcommand{\rdisc}{\textnormal{q-disc}}

\newtheorem{definition}{Definition}
\newtheorem{theorem}{Theorem}
\newtheorem{corollary}{Corollary}
\newtheorem{lemma}{Lemma}

\newtheorem{assumption}{Assumption}

\usepackage{hyperref}
\usepackage[a4paper,top=3cm,bottom=2cm,left=3cm,right=3cm,marginparwidth=1.75cm]{geometry}

\title{A Theory of Output-Side Unsupervised Domain Adaptation}
\author{Tomer Galanti, Lior Wolf}
\date{}
\begin{document}

\tikzstyle{b} = [rectangle, draw, fill=blue!20, node distance=3cm, text width=6em, text centered, rounded corners, minimum height=4em, thick]
\tikzstyle{c} = [rectangle, draw, inner sep=0.5cm, dashed]
\tikzstyle{l} = [draw, -latex',thick]

\maketitle

\begin{abstract}

When learning a mapping from an input space to an output space, the assumption that the sample distribution of the training data is the same as that of the test data is often violated. Unsupervised domain shift methods adapt the learned function in order to correct for this shift. Previous work has focused on utilizing unlabeled samples from the target distribution. We consider the complementary problem in which the unlabeled samples are given post mapping, i.e., we are given the outputs of the mapping of unknown samples from the shifted domain. Two other variants are also studied: the two sided version, in which unlabeled samples are give from both the input and the output spaces, and the Domain Transfer problem, which was recently formalized. In all cases, we derive generalization bounds that employ discrepancy terms.
\end{abstract} 

\section{Introduction}

In the unsupervised domain adaptation problem~\cite{Crammer:2008:LMS:1390681.1442790,DBLP:conf/colt/MansourMR09,DBLP:journals/ml/Ben-DavidBCKPV10}, the algorithm trains a hypothesis on a source domain and the hypothesis is tested on a similar yet different target domain. The algorithm is aided with a labeled dataset of the source domain and an unlabeled dataset of the target domain. The conventional approach to dealing with this problem is to learn a feature map that (i) enables accurate  classification in the source domain and (ii) captures meaningful invariant relationships between the source and target domains. 

The standard unsupervised domain adaptation problem does not capture the scenario in which the orientation is performed with the output. In such scenarios, the set of unlabeled samples of the target domain is replaced by  the labels of such a set. In other words, the learning algorithm receives a labeled dataset in the source domain and a dataset of the outputs of the target function to learn on random samples from the target distribution. As far as we know, this problem is novel despite being ecological (i.e., appearing naturally in the real-world), widely applicable and likely to take place in cognitive reasoning.

As a motivating example, consider a system that learns to map data about houses to their market prices. The system is then asked to adapt to the segment of the market in which prices are in a certain range. Houses out of this range can still serve as valuable examples for recovering the ``regression coefficients'' of streets, neighborhoods and the number of rooms. Adaptation is expected to outperform a simple filtering of the dataset. 

The new problem, which we call Output-Side Domain Adaptation (ODA). Underlies real-world AI challenges that humans deal with. Consider an AI agent that learns how to program Java through examples of programming challenges (specifications) and their solutions (Java code). The agent is also presented with a large corpus of C$\#$ code and is required to adapt to C$\#$ programming. 

The main tool that we apply in order to analyze the ODA problem is discrepancy, which is already in wide use in the study of standard unsupervised domain adaptation~\cite{DBLP:journals/ml/Ben-DavidBCKPV10,DBLP:conf/alt/Mansour09}. Recently, ~\cite{icml2015_ganin15, Ganin:2016:DTN:2946645.2946704} tied the notion of discrepancy to the GAN method~\cite{DBLP:conf/nips/GoodfellowPMXWOCB14}. It was shown that GANs can implement these discrepancies very effectively and that the combination of GANs with domain adaptation led to an improved accuracy in comparison with other recent approaches.

In addition to ODA we also study the two-sided version, in which we are given two sets of unmatched samples in both the source and the target domain: one for input samples and one for output samples. Interestingly, the generalization bound we derive motivates the recent CoGAN method of~\cite{cogan}. We then employ the same tools in order to study a third problem in which the output distribution is given in an unsupervised manner, namely the problem of unsupervised cross domain sample generation~\cite{DBLP:journals/corr/TaigmanPW16}. In this problem, two unsupervised sets are provided, one containing a set of samples from the input domain and another from the output domain. In addition, some metric that can compare samples between the two domains is given. The task is to build a mapping between the two domains such that this metric is minimized. Similar to the ODA problem, we use discrepancies in order to derive generalization bounds for this problem thus providing theoretical foundations to the DTN algorithm of~\cite{DBLP:journals/corr/TaigmanPW16}.

\section{Preliminaries} 
Our work has close ties to the classical work on domain adaptation, which we review below. We also review GANs through the lens of discrepancy.

\subsection{Unsupervised domain adaptation}

A {\em Domain Adaptation Setting} is specified by a tuple $(\mathcal{H}_1, \mathcal{H}_2,\mathcal{Z},\ell)$, consisting of: a set of feature maps $\mathcal{H}_1 = \{h_1:\mathcal{X} \rightarrow \mathcal{F}\}$, a set of classifiers, $\mathcal{H}_2 = \{h_2:\mathcal{F} \rightarrow \mathcal{Y}\}$, a set of samples $\mathcal{Z} = \mathcal{X}\times \mathcal{Y}$ and a loss function $\ell: \mathcal{Y}\times\mathcal{Y} \rightarrow \mathbb{R}_+$. In this model, the hypothesis class is,
\begin{small}
\begin{equation}
\mathcal{H} := \mathcal{H}_2 \circ \mathcal{H}_1 := \left\{g\circ f \Big\vert f\in \mathcal{H}_1, g \in \mathcal{H}_2 \right\}
\end{equation}
\end{small}
Each hypothesis $h \in \mathcal{H}$ is decomposed into a feature map $f$ and a classifier $g$. The feature map $f$ takes inputs $x \in \mathcal{X} \subset \mathbb{R}^{d_1}$ and represents them as vectors in the feature space, $\mathcal{F}$. Subsequently, the classifier, $g$, takes inputs from the feature space and maps them to labels in $\mathcal{Y} \subset \mathbb{R}^{d_2}$. 

We assume a source domain and a target domain (a distribution over $\mathcal{X}$ along with a function $\mathcal{X} \rightarrow \mathcal{Y}$) $(D_S,y_S)$ and $(D_T,y_T)$ (respectively). The fitting of each hypothesis $h \in \mathcal{H}$ is measured by the {\em Target Generalization Risk}, $R_{D_T}[h,y_T]$. Where, the {\em Generalization Risk} is defined as 
\begin{small}
$R_{D}[h_1,h_2]=\mathbb{E}_{x\sim D_T}\left[\ell(h_1(x),h_2(x))\right]$.
\end{small}
Here, $\mathcal{H}$, $\mathcal{Z}$, and $\ell$ are known to the learner. The distributions $D_S$, $D_T$ and the target function $y_T:\mathcal{X} \rightarrow \mathcal{Y}$ are unknown to the learner. The goal of the learner is to pick $h \in \mathcal{H}$ that optimizes
\begin{small}
$\inf_{h \in \mathcal{H}} R_{D_T}[h,y_T]$.
\end{small}
Since the target function, $y_T$, and the target distribution, $D_T$, are unknown, this quantity cannot be computed directly. 

In most of the machine learning literature, the learning algorithm is being trained and tested on the target distribution. In domain adaptation, the learning algorithm is being trained on labeled samples from the source domain and unlabeled samples from the target domain. Formally, the learner is provided with the following two datasets, 
\begin{small}
\begin{equation}
\begin{aligned}
\{(x_i,y_S(x_i))\}^m_{i=1} \text{ such that } x_i \stackrel{\textnormal{i.i.d}}{\sim} D_S  \\
\{x_i\}^{n}_{i=1} \text{ such that } x_i \stackrel{\textnormal{i.i.d}}{\sim} D_T.
\end{aligned}
\end{equation}
\end{small}
See Fig.~\ref{fig:illustration}(a) for an illustration.

In many machine learning settings that require minimizing a generalization risk, it is approximated with its corresponding {\em Empirical Risk}
\begin{small}
$\hat R_{D}[h,y]=\frac{1}{m}\sum^{m}_{i=1}\ell(h(x_i),y(x_i))$. 
\end{small}
For a dataset $\{(x_i,y(x_i))\}^m_{i=1}$ such that $x_i \stackrel{\textnormal{i.i.d}}{\sim} D$. In several domain adaptation settings, the algorithm minimizes the {\em Source Generalization Risk} and the distance between the two domains. In order to approximate the source generalization risk, we make use of the {\em Source Empirical Risk}, $\hat R_{D_S}[h,y_S]$.

\begin{figure*}[t]
\centering
\begin{tabular}{cc}
\begin{tabular}{|c|C{2.6cm}C{2.8cm}|}
\hline
      & Input & Output\\
      \hline
    $1^{st}$  &   \begin{small}$\{x_i\sim D_S\}$\end{small} & \begin{small}$\{y_S(x_i)\}$\end{small} \\ 
    $2^{st}$  &   \begin{small}$\{x_j\sim D_T\}$\end{small}\tikzmark{a} & \tikzmark{b}\\
\hline
  \end{tabular}
  
  &
  \begin{tabular}{|c|C{2.6cm}C{2.8cm}|}
\hline
       & Input & Output\\
      \hline
    $1^{st}$  &   \begin{small}$\{x_i\sim D_S\}$\end{small} & \begin{small}$\{y_S(x_i)\}$\end{small} \\ 
    $2^{st}$  &   \tikzmark{a} & \begin{small}$\{y_T(x_j) | x_j\sim D_T\}$\end{small}\tikzmark{b}\\
   \hline
  \end{tabular}
  
  \\

    (a) & (b)\\
     \begin{tabular}{|c|C{2.6cm}C{2.8cm}|}
\hline
       & Input & Output\\
      \hline
    $1^{st}$  &   \begin{small}$\{x_i\sim D_1\}$\end{small} & \begin{small}$\{y_1(x_j) | x_j \sim D_1\}$\end{small} \\ 
    $2^{st}$  &\begin{small}$\{x_k\sim  D_2\}$\end{small}\tikzmark{a} & \begin{small}$\{y_2(x_l) | x_l\sim D_2\}$\end{small}\tikzmark{b}\\
   \hline
  \end{tabular}
  &
  \begin{tabular}{|c|C{2.6cm}C{2.8cm}|}
\hline
       & Input & Output\\
      \hline
    $1^{st}$  &   \begin{small}$\{x_i\sim D_1\}$ \end{small}&  \\ 
    $2^{st}$  &   \tikzmark{a} & \tikzmark{b}\begin{small}$\{y(x_j)| x_j\sim D_2\}$\end{small}\\
    \hline
  \end{tabular}

  \\
  (c) & (d)\\
\end{tabular}
\caption{\label{fig:illustration} A comparison of the various domain shift models discussed in this work. (a) The conventional unsupervised domain adaptation problem. The algorithm learns a function $y_T$ from samples $\{(x_i\sim D_S,y_S(x_i))\}^{m}_{i=1}$ and $\{x_i\sim D_T\}^{n}_{i=1}$. (b) The output-side unsupervised domain adaptation problem. Instead of $\{x_j\sim D_T\}^{n}_{i=1}$, the algorithm is provided with $\{y_T(x_j)\sim D^y_T\}^{n}_{i=1}$. (c) In the two sided variant, the goal is to learn $y_T$ given samples $\{x_i\sim D_1\}^{m_1}_{i=1}$, $\{y_1(x_j) | x_j \sim D_1\}^{m_2}_{j=1}$,  $\{x_k\sim D_2\}^{n_1}_{k=1}$ and $\{y_2(x_l) | x_l \sim D_2\}^{n_2}_{l=1}$.  (d) The unsupervised domain transfer problem. In this case, the algorithm learns a function $y$ and is being tested on $D_1$. The algorithm is aided with two datasets: $\{x_i \sim D_1\}^{m}_{i=1}$ and $\{y(x_j) \sim D^y_2\}^{n}_{j=1}$. }
\end{figure*}

\paragraph{Distances between distributions}

Different methodologies for domain adaptation exist in the literature. In the unsupervised domain adaptation model, the learning algorithm uses the source dataset in order to learn a hypothesis that fits it and the unlabeled target dataset in order to measure and restrict closeness between the source and target distributions. In \cite{DBLP:conf/colt/MansourMR09} and \cite{DBLP:journals/ml/Ben-DavidBCKPV10}, it is assumed that $\mathcal{H}_1$ consists of only one representation function $f$ such that $f \circ D_S$ and $f \circ D_T$ are close in some sense. In \cite{icml2015_ganin15} the algorithm learns $f \in \mathcal{H}_1$ such that $f \circ D_S$ and $f \circ D_T$ are close and $h = g\circ f$ fits the source task well. 
Therefore, a critical component in domain adaptation is the ability to restrict the source and target domains to be close by some distance. Different definitions of distance were suggested in the literature. For example, in the context of binary classification, \cite{DBLP:journals/ml/Ben-DavidBCKPV10} explained that the $\mathcal{H}$-divergence distance is more appealing than the total variation (TV) distance. In addition, \cite{DBLP:conf/colt/MansourMR09} extended the discussion regarding the $\mathcal{H}$-divergence distance to the more general notion of discrepancy distance in order to deal with regression tasks.  

\begin{definition}[Discrepancy distance]\label{def:discrepancy}
Let $\mathcal{C}$ be a class of functions from $A$ to $B$ and let $\ell: B\times B  \rightarrow \mathbb{R}_+$ be a loss function over $B$. The discrepancy distance $\disc_{\mathcal{C}}$ between two distributions $D_1$ and $D_2$ over $A$ is defined as follows,
\begin{small}
\begin{equation}\label{eq:disc}
\disc_{\mathcal{C}}(D_1,D_2) = \sup_{c_1,c_2 \in \mathcal{C}}\Big\vert R_{D_1}[c_1,c_2]-R_{D_2}[c_1,c_2]\Big\vert
\end{equation}
\end{small}
\end{definition}

\paragraph{Generalization bounds}

We next review the bounds provided by \cite{DBLP:conf/colt/MansourMR09} and \cite{DBLP:journals/ml/Ben-DavidBCKPV10}. In the following sections, we will compare them to the results proposed in the current work.
The bounds are presented in a slightly modified version in order to support such a comparison and are illustrated in Fig.~\ref{fig:da}.

\begin{theorem}[\cite{DBLP:conf/colt/MansourMR09}]\label{thm:mansour} Let $\mathcal{H} = \mathcal{H}_2 \circ \mathcal{H}_1$. Assume that the loss function $\ell$ is symmetric and obeys the triangle inequality. Then, for any hypothesis $h = g\circ f \in \mathcal{H}$, the following holds
\begin{small}
\begin{equation}
\begin{aligned}
R_{D_T}[h,y_T] \leq& R_{D_S}[h,h^{*}_S]+R_{D_T}[h^*_T,y_T] \\
&+ R_{D_S}[h^*_S,h^*_T]\\
&+ \disc_{\mathcal{H}_2}(f \circ D_S, f\circ D_T)
\end{aligned}
\end{equation}
\end{small}
Here, $h^*_T := g^{*}_T \circ f := \arg\min_{h \in \mathcal{H}_2 \circ f} R_{D_T}\left[h,y_T\right]$ and $h^*_S := g^{*}_S \circ f$ is the same for the source domain $(D_S,y_S)$.
\end{theorem}

We also provide a general variation of the original bound proposed by \cite{DBLP:journals/ml/Ben-DavidBCKPV10}.

\begin{theorem}[\cite{DBLP:journals/ml/Ben-DavidBCKPV10}]\label{thm:bendavid} Let $(\mathcal{H}_1, \mathcal{H}_2, \ell, \mathcal{Z})$ be a binary classification domain adaptation setting (i.e, $\ell$ is the 0-1 loss and $\mathcal{Y}=\{0,1\}$). Assume that $y:=y_S=y_T$. Then, for any hypothesis $h = g\circ f \in \mathcal{H}$, 
\begin{small}
\begin{equation}
\begin{aligned}
R_{D_T}[h,y] \leq& R_{D_S}[h,y]+ \disc_{\mathcal{H}_2}(f \circ D_S, f\circ D_T) +\lambda
\end{aligned}
\end{equation}
\end{small}
Where, \begin{small}$\lambda=\min_{g\in \mathcal{H}_2}\left\{\left[R_{D_T}[g\circ f,y] + R_{D_S}[g\circ f,y]\right] \right\}$\end{small}.
\end{theorem}

\subsection{Unsupervised domain adaptation and GANs}
\label{sec:pregan}

Generative Adversarial Networks (GANs) were first proposed by~\cite{DBLP:conf/nips/GoodfellowPMXWOCB14}. The idea behind GANs involves learning a generative model through an adversarial process, in which two models are trained simultaneously: a generative model $f \in \mathcal{H}_G$
that captures the data distribution, and a discriminative model $d \in \mathcal{H}_D$ that estimates the probability that a sample came from the training data rather than $f$. The training
procedure for $f$ is to maximize the probability of $d$ making a mistake. In other words, $f$ and $d$ play the following two-player minimax game, 
\begin{small}
\begin{equation}\label{eq:GAN}
\begin{aligned}
\min_{f \in \mathcal{H}_G}\max_{d \in \mathcal{H}_D}& \mathbb{E}_{z \sim D_S}[\log(1-d(f(z)))] \\
&+ \mathbb{E}_{x \sim D_T}[\log(d(x))],
\end{aligned}
\end{equation}
\end{small}
where $d(x)$ is the probability that the classifier $d$ assigns to sample $x$ being a ``real'' sample from the distribution $D_T$, rather than a ``fake'' sample generated by $f$ to some random input $z$ from the distribution $D_S$.

Both Thm.~\ref{thm:mansour} and~\ref{thm:bendavid}  motivate the following optimization criterion, which was investigated by~\cite{icml2015_ganin15}.
\begin{small}
\begin{equation}\label{eq:alg1}
\begin{aligned}
\arg\min_{f,g} \hat R_{D_S}[g\circ f,y]+\disc_{\mathcal{H}_2}(f\circ \hat D_S, f\circ \hat D_T)
\end{aligned}
\end{equation}
\end{small}

It is shown that GANs and the discrepancy distance are closely tied with each other. Specifically the two classifier $c_1$ and $c_2$ in Eq.~\ref{eq:disc} can be replaced with a binary classifier $d$ from the class $\mathcal C\Delta \mathcal C:=\{[c_1(x) \neq c_2(x)]  | c_1,c_2 \in \mathcal{C}\}$, i.e., the class of functions that check equality between pairs of functions in $\mathcal{C}$.

In the general case, one way to connect GANs and discrepancy is through $f$-divergences as shown in~\cite{DBLP:conf/nips/NowozinCT16}. Specifically, both GANs and discrepancies are special cases of a lower bound that is due to~\cite{DBLP:journals/tit/NguyenWJ10} for $f$-divergences between distributions. In particular, discrepancy is the instantiation of the bound for the TV-distance and GAN is the analog for another specific form of $f$-divergence.

\begin{figure}[t]
\begin{center}
\begin{tikzpicture}[auto,scale=0.8]
    \node (rep1) {$f\circ D_S$};
    \node (rep2) at ([shift={(0,-1.5)}] rep1) {$f\circ D_T$};
    \node (out1) at ([shift={(3.5,0)}] rep1) {$h \circ D_S$};  
    \node (in1) at ([shift={(-3.5,0)}] rep1)   {$D_S$};
    \node (in2) at ([shift={(-3.5,-1.5)}] rep1) {$D_T$};
    \node (out3) at ([shift={(3.5,1.5)}] rep1) {$D^y_S$};
	\path[-stealth] (rep1) edge (rep2) ;
    \path[-stealth] (rep2) edge node [right] {disc} (rep1) ;
    \path[-stealth] (out3) edge node [right] {risk} (out1) ;
    \path[-stealth] (out1) edge (out3) ;
    \draw[->] (rep1)  -- (out1) node[midway] {$g$};

    \draw[->] (in1) -- (rep1) node[midway] {$f$};
    \draw[->] (in2) -- (rep2) node[midway] {$f$};
    \path [l] (in1.north) -- ++(0,1.1633)  -- node[pos=0.55] {$y_S$} (out3.west) ;
\end{tikzpicture}
\end{center}
\caption{\label{fig:da} Unsupervised domain adaptation. Each node contains a distribution, the horizontal edges denote the mappings between the distributions and the learned function is $h = g \circ f$. The vertical edges denote the discrepancy between the the two distributions $f\circ D_S$ and $f\circ D_T$ and the risk between $y$ and $h$ on $D_S$.}
\end{figure}
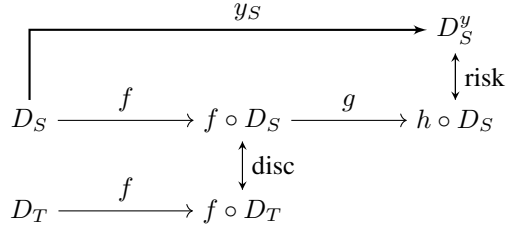

\section{Output-side domain adaptation}

We next present the problem setup of output-side domain adaptation, which is a new variant of unsupervised domain shift problems.

The major difference between the conventional unsupervised domain adaptation and the new variant  is that instead of letting the learner access a dataset of i.i.d unlabeled target instances (in addition to a dataset of i.i.d labeled samples from the source domain), it has access to a dataset of output labels that correspond to i.i.d instances from the target distribution. Formally, the learner is provided with the following two datasets, as illustrated in Fig.~\ref{fig:illustration}(b)
\begin{small}
\begin{equation}
\begin{aligned}
\{(x_i,y_S(x_i))\}^m_{i=1} &\text{ such that } x_i \stackrel{\textnormal{i.i.d}}{\sim} D_S  \\
\{y_T(x_j)\}^{n}_{j=1} &\text{ such that } x_j \stackrel{\textnormal{i.i.d}}{\sim} D_T
\end{aligned}
\end{equation}
\end{small}
We will use the notation $t \sim D^y_T := y_T \circ D_T$ to denote $t=y_T(x)$ where $x \sim D_T$. $D^y_S:= y_S \circ D_S$ is similarly defined, and, in general, we will use the notation $p \circ D$ to denote the distribution of $p(x)$ where $x \sim D$.  

In order to model the situation, we decompose the hypothesis $h = g\circ f \in \mathcal{H}_2 \circ \mathcal{H}_1$ as done in conventional domain adaptation. In output-side domain adaptation, we present an additional functions class $\mathcal{H}'_2$ and we learn a pseudo-inverse $\hat g$ of $g$ taken from $\mathcal{H}'_2$. This function helps in recovering the feature representation of a given output.

\paragraph{Assumptions} Our generalization bounds will rely on two assumptions. They are the universal Lipschitzness and the factor triangle inequality.

We begin with the definition of a Lipschitz functions class.
\begin{definition}[Lipschitz hypothesis class] Let $\ell:B \times B \rightarrow \mathbb{R}_+$ be a loss function over $B$. Let $\mathcal{C}$ be a class of functions $c:A \rightarrow B$. 
\begin{itemize}
\item A function $c\in \mathcal{C}$ is Lipschitz with respect to $\ell$, if there is a constant $L>0$ such that: $\forall a_1, a_2 \in A: \ell(c(a_1),c(a_2)) \leq L \cdot \ell(a_1,a_2) $.
\item $\mathcal{C}$ is a universal Lipschitz hypothesis class with respect to $\ell$, if all function $c\in\mathcal{C}$ are Lipschitz with some universal constant $L>0$.
\item $\mathcal{C}$ is a universal Bi-Lipschitz hypothesis class with respect to $\ell$, if every function $c\in\mathcal{C}$ is invertible and both $\mathcal{C}$ and $\mathcal{C}^{-1}=\{c^{-1}:c \in \mathcal{C}\}$ are universal Lipschitz hypothesis classes.
\end{itemize}
\end{definition}

\begin{assumption}[Universal Lipschitzness]\label{assmp:lipsch} We assume that $\mathcal{H}_2$ and $\mathcal{H}'_2$ are universal Lipschitz hypothesis classes with respect to the loss function $\ell: \mathcal{Y}\times\mathcal{Y}  \rightarrow \mathbb{R}_+$.
\end{assumption}

The assumption holds, for example, for $\ell$ that is a squared loss or the absolute loss, where $\mathcal{H}_2$ and $\mathcal{H}'_2$ consist of feedforward neural networks with the activation function 
\begin{small}
$\textnormal{PReLU}_\alpha(x)=
\left\{
	\begin{array}{ll}
		x  & \mbox{if } x \geq 0 \\
		\alpha x & \mbox{if } x < 0
	\end{array}
\right.$ 
\end{small} with parameter $\alpha \geq 0$ and the weight matrix of each layer has a norm $\in [a,b]$ such that $b > a > 0$. 

\begin{assumption}[Factor triangle inequality]\label{assmp:triangle} Let $\ell: \mathcal{Y}\times\mathcal{Y}  \rightarrow \mathbb{R}_+$ be the loss function. We assume that $\ell$ obeys a factor-triangle-inequality, i.e,
\begin{small}
\begin{equation}
\begin{aligned}
\exists K>0:\forall& y_1,y_2,y_3 \in \mathcal{Y}:\\
&\ell(y_1,y_3) \leq K \left[\ell(y_1,y_2)+\ell(y_2,y_3)\right] 
\end{aligned}
\end{equation}
\end{small}
\end{assumption}

The second assumption allows us to address common losses. the absolute loss, $\ell(a,b)= |a-b|$, satisfies the assumption with constant $K=1$ and the squared loss, $\ell(a,b)=|a-b|^2$, satisfies it with $K=3$. 

\subsection{Generalization bounds}

This section presents generalization bounds for output-side domain adaptation given in terms of the discrepancy distance. 

\begin{theorem}\label{thm:main1} If Assumptions~\ref{assmp:lipsch} and~\ref{assmp:triangle} hold, then for all $h=g\circ f \in \mathcal{H}$ and $\hat g\in \mathcal{H}'_2$,
\begin{small}
\begin{equation}\label{eq:main1}
\begin{aligned}
R_{D_T}&[h,y_T] 
\lesssim R_{D_S}[h, h^*_S] +R_{D_S}[h^*_S,h^*_T] +R_{D_T}\left[h^{*}_{T}, y_T \right]\\
&+R_{\hat g \circ D^y_T}\left[\hat g_T \circ g^{*}_{T}, \Id \right]+R_{f \circ D_T}\left[\hat g_T \circ g^{*}_{T}, \Id \right]\\
&+ R_{D^y_T}[g\circ \hat g, \Id] + \disc_{\mathcal{H}_2}(f\circ D_S,\hat g\circ D^y_T)
\end{aligned}
\end{equation}
\end{small}
Here, $h^*_T := g^{*}_T \circ f := \arg\min_{h \in \mathcal{H}_2 \circ f} R_{D_T}\left[h,y_T\right]$ and $h^*_S := g^{*}_S \circ f$ be the same for the source domain $(D_S,y_S)$. In addition, we denote 
$\hat g_T = \arg\min_{\bar g \in \mathcal{H}'_2} \left\{R_{\hat g \circ D^y_T}\left[\bar g \circ g^{*}_{T}, \Id \right]+R_{f \circ D_T}\left[\bar g \circ g^{*}_{T}, \Id \right]\right\}$.
\end{theorem}

\begin{proof}
By the factor triangle inequality,
\begin{small}
\begin{equation}\label{eq:prf1eq1}
\begin{aligned}
R_{D_T}&[h,y_T] \\
\lesssim& R_{D_T}\left[g\circ \hat g_{T}\circ y_T, y_T\right] + R_{D_T}\left[h, g \circ \hat g_{T}\circ y_T\right] \\ 
=& R_{D^y_T}\left[g\circ \hat g_{T}, \Id\right] + R_{D_T}\left[h, g \circ \hat g_{T}\circ y_T\right] \\
\end{aligned}
\end{equation}
\end{small}
Since $\mathcal{H}_2$ is a universal Lipschitz hypothesis class and by the factor triangle inequality, 
\begin{small}
\begin{equation}
\begin{aligned}
R_{D^y_T}&\left[g\circ \hat g_{T}, \Id\right] \lesssim R_{D^y_T}\left[g\circ \hat g_{T}, g\circ \hat g\right] + R_{D^y_T}[g\circ \hat g, \Id]\\
\lesssim& R_{D^y_T}\left[\hat g_{T}, \hat g\right]+ R_{D^y_T}[g\circ \hat g, \Id]\\
\end{aligned}
\end{equation}
\end{small}
Since $\mathcal{H}'_2$ is a universal Lipschitz hypothesis class and by the factor triangle inequality, 
\begin{small}
\begin{equation}
\begin{aligned}
R_{D^y_T}&\left[g\circ \hat g_{T}, \Id\right] 
\lesssim  R_{\hat g \circ D^y_T}\left[\hat g_{T}\circ g, \Id\right]+R_{D^y_T}\left[\hat g_{T}\circ g \circ \hat g, \hat g_T\right]\\
&+ R_{D^y_T}[g\circ \hat g, \Id]\\
\lesssim & R_{\hat g \circ D^y_T}\left[\hat g_{T}\circ g, \Id\right] + R_{D^y_T}[g\circ \hat g, \Id]\\
\lesssim& R_{\hat g\circ D^y_T}\left[\hat g_{T} \circ g, \hat g_{T}\circ g^{*}_{T}\right]+R_{\hat g\circ D^y_T}\left[\hat g_{T}\circ g^{*}_{T},\Id \right]+ R_{D^y_T}[g\circ \hat g, \Id]\\
\lesssim& R_{\hat g\circ D^y_T}\left[g,  g^{*}_{T}\right]+R_{\hat g\circ D^y_T}\left[\hat g_{T}\circ g^{*}_{T},\Id \right]+ R_{D^y_T}[g\circ \hat g, \Id]\\
\end{aligned}
\end{equation}
\end{small}
By the definition of discrepancy,
\begin{small}
\begin{equation}\label{eq:prf1eq2}
\begin{aligned}
R_{\hat g\circ D^y_T}&\left[g,  g^*_{T}\right] \lesssim  R_{f\circ D_S}\left[ g, g^*_{T}\right] +\disc_{\mathcal{H}_2}(f\circ D_S, \hat g\circ D^y_T) \\
=& R_{D_S}\left[ h, h^{*}_{T}\right] +\disc_{\mathcal{H}_2}(f\circ D_S, \hat g\circ D^y_T)\\
\lesssim& R_{D_S}\left[ h, h^{*}_{S}\right]+
R_{D_S}\left[ h^*_S, h^{*}_{T}\right]\\
&+\disc_{\mathcal{H}_2}(f\circ D_S, \hat g\circ D^y_T)
\end{aligned}
\end{equation}
\end{small}
In addition, since $\mathcal{H}_2$ and $\mathcal{H}'_2$ are universal Lipschitz hypothesis classes and by the factor triangle inequality,
\begin{small}
\begin{equation}\label{eq:prf1eq3}
\begin{aligned}
R_{D_T}&\left[h, g \circ \hat g_{T}\circ y_T\right]\lesssim R_{D_T}\left[f, \hat g_{T}\circ y_T \right]\\
\lesssim& R_{D_T}\left[\hat g_{T} \circ h^{*}_{T}, \hat g_{T}\circ y_T \right]+ R_{D_T}\left[\hat g_{T} \circ h^{*}_{T}, f \right]\\
\lesssim& R_{D_T}\left[h^{*}_{T}, y_T \right]+ R_{f\circ D_T}\left[\hat g_{T} \circ g^{*}_{T}, \Id \right]\\
\end{aligned}
\end{equation}
\end{small}
Combining Eqs.~\ref{eq:prf1eq1},~\ref{eq:prf1eq2},~\ref{eq:prf1eq3} leads to the desired bound.
\end{proof}

The bound is illustrated in Fig.~\ref{fig:ODAfig}.  Comparing the bound in Thm.~\ref{thm:main1} to the bound in Thm.~\ref{thm:mansour}, we note that the two bounds seem to be very similar to each other. In both cases, the target generalization risk, $R_{D_T}[h,y_T]$, is upper bounded by the sum between the source estimation risk with respect to the best source hypothesis, $R_{D_S}[h,h^*_S]$, the discrepancy between the distributions over the feature space and an unmeasurable constant, $R_{D_T}\left[h^{*}_{T}, y_T \right]+R_{D_S}\left[h^{*}_{S}, h^*_T \right]$. There are two main differences between the two bounds. The first is that the discrepancies differ. In the bound in Thm.~\ref{thm:mansour}, the term is $\disc_{\mathcal{H}_2}(f \circ D_S, f \circ D_T)$ while in Thm.~\ref{thm:main1}, the term is the analogue in the output-side domain adaptation setting, $\disc_{\mathcal{H}_2}(f\circ D_S, \hat g\circ D^y_T)$. In addition, in the bound in Thm.~\ref{thm:main1}, there are three additional invertibility terms. The first two terms measure the invertibility of $g^{*}_T$, i.e, $R_{\hat g \circ D^y_T}\left[\hat g_T \circ g^{*}_{T}, \Id \right]+R_{f \circ D_T}\left[\hat g_T \circ g^{*}_{T}, \Id \right]$. The third term measures the invertibility of of $g$, i.e, $R_{D^y_T}[g\circ \hat g, \Id]$.

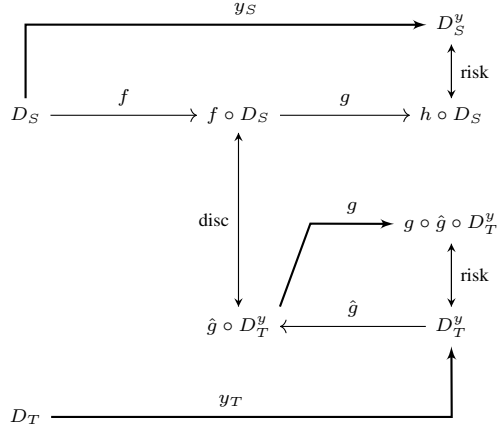
\begin{figure}[t]
\begin{center}
\begin{tikzpicture}[auto,scale=0.8]
    \node (rep1) {\begin{scriptsize}$f\circ D_S$\end{scriptsize} };
    
    \node (rep2) at ([shift={(0,-3.5)}] rep1) 
    {\begin{scriptsize}$\hat g\circ D^y_T$\end{scriptsize} };
    
    \node (out1) at ([shift={(3.5,0)}] rep1) 
    {\begin{scriptsize}$h \circ D_S$\end{scriptsize} };
    
    \node (out2) at ([shift={(3.5,-3.5)}] rep1) {\begin{scriptsize}$D^y_T$\end{scriptsize}};
    
    \node (in1) at ([shift={(-3.5,0)}] rep1)   
    {\begin{scriptsize}$D_S$\end{scriptsize} };
    
    \node (in2) at ([shift={(-3.5,-5)}] rep1) 
    {\begin{scriptsize}$D_T$\end{scriptsize}};
    
    \node (out3) at ([shift={(3.5,1.5)}] rep1) 
    {\begin{scriptsize}$D^y_S$\end{scriptsize}};
    
    \node (out4) at ([shift={(3.5,-1.8)}] rep1) 
    {\begin{scriptsize}$g \circ \hat g \circ D^y_T$\end{scriptsize}};
    
	\path[-stealth] (rep1) edge node [left] 
    {\begin{scriptsize}disc\end{scriptsize}} (rep2) ;
    \path[-stealth] (rep2) edge (rep1) ;
    \path[-stealth] (out3) edge node [right] 
    {\begin{scriptsize}risk\end{scriptsize}} (out1) ;
    \path[-stealth] (out1) edge (out3) ;
    
    \path[-stealth] (out4) edge node [right] 
    {\begin{scriptsize}risk\end{scriptsize}} (out2) ;
    \path[-stealth] (out2) edge (out4) ;
    
    \draw[->] (rep1)  -- (out1) node[midway] 
    {\begin{scriptsize}$g$\end{scriptsize}};
    \draw[->] (out2) -- (rep2) node[midway,above] 
    {\begin{scriptsize}$\hat g$\end{scriptsize}};

    \draw[->] (in1) -- (rep1) node[midway] 
    {\begin{scriptsize}$f$\end{scriptsize}};

    \path [l] (rep2.north east) -- ++(0.5,1.375)   -- node[pos=0.5] 
    {\begin{scriptsize}$g$\end{scriptsize}}
    (out4.west) ;

    \path [l] (in2.east) -- ++(6.585,0) node[pos=0.45] 
    {\begin{scriptsize}$y_T$\end{scriptsize}}  
    -- (out2.south) ;
    \path [l] (in1.north) -- ++(0,1.22)  -- node[pos=0.55] 
    {\begin{scriptsize}$y_S$\end{scriptsize}} 
    (out3.west) ;
    
\end{tikzpicture}
\end{center}
\caption{\label{fig:ODAfig} Output-side domain adaptation. Similarly to Fig.~\ref{fig:da}, the learned function is $h = g \circ f$. The vertical edges stand for the discrepancy between the the two distributions $f\circ D_S$ and $\hat g\circ D_T$, the risk between $h$ and $y_S$ on the source distribution, $D_S$ and the risk between $g \circ \hat g$ and $\Id$ on $D^y_T$.}
\end{figure}

\subsection{Analogy based adaptation}

The bound presented above contains the risk between the two hypotheses $h^*_S$ and $h^*_T$ for the distribution $D_S$. However, the image of $y_S$ and the image of $y_T$ might be completely disjoint, making this risk unmanageable. This is also true for conventional unsupervised domain adaptation.

Consider, for example, the programing languages example in the introduction. We can set $D_S=D_T$ to be a fixed distribution of program specifications. Since the Java and the C\# programs compile successfully on the respective compiler, the adaptation results in disjoint $y_S$ and $y_T$.

In order to model this situation, we decompose a hypothesis $h = g\circ f$ such that $g=a^{-1} \circ b$. The component $a$ serves as an adapter that maps target domain outputs to source domain outputs and is assumed to be invertible. In addition, we also learn an invertible function $b$, that maps $\mathcal{Y}$ to the feature space $\mathcal{F}$. Similarly, $\mathcal{H} = \mathcal{H}_2 \circ \mathcal{H}_1$ and $\mathcal{H}_2 = \mathcal{H}^{-1}_3 \circ \mathcal{H}_4$. Here, $\mathcal{H}_1$ is the hypothesis class (e.g., $h$), $\mathcal{H}_2$ is the set of classifiers (e.g., $g$),  $\mathcal{H}_3$ is a set of adapters (e.g., $a$) and $\mathcal{H}_4$ is a set of output-side to feature space mappings (e.g., $b$).  

For simplicity, we assume that $\mathcal H_2$ is a class of invertible functions. This can be relaxed, similar to what was done in Thm.~\ref{thm:main1}, at the cost of adding more risk terms.

\begin{theorem}\label{thm:main2} If Assumption~\ref{assmp:triangle} holds and $\mathcal{H}_2, \mathcal{H}_3$ are a Bi-Lipschitz hypothesis classes, then for all $h=g\circ f \in \mathcal{H}$ where $g = a^{-1}\circ b \in \mathcal{H}_2$,
\begin{small}
\begin{equation}
\begin{aligned}
R_{D_T}&[h,y_T] \lesssim R_{D_S}[a\circ h,y_S] + R_{D_S}[a\circ h^*_T,h^*_S]  \\
&+R_{D_S}[h^*_S,y_S] +R_{D_T}[h^*_T,y_T]+\disc_{\mathcal{H}^{-1}_4}(a\circ D^y_T, D^y_S) \\
\end{aligned}
\end{equation}
\end{small}
Here, $h^*_T := g^{*}_T \circ f$ where $g^{*}_T := a^{-1}\circ b^*_T$ such that 
\[
b^*_T :=  \arg\min_{b \in \mathcal{H}_4} R_{D_T}\left[a^{-1}\circ b \circ f ,y_T\right]
\]
In addition, $h^*_S$, $g^{*}_S$ and $b^{*}_S$ are the same for the source domain $(D_S,y_S)$.  
\end{theorem}

\begin{proof} 
By the factor triangle inequality,
\begin{small}
\begin{equation}\label{eq:prf1seq1}
\begin{aligned}
R_{D_T}&[h,y_T] \\
\lesssim& R_{D_T}\left[g\circ (g^{*}_{T})^{-1}\circ y_T, y_T\right] + R_{D_T}\left[h, g \circ  (g^{*}_{T})^{-1}\circ y_T\right] \\ 
=& R_{D^y_T}\left[g\circ (g^{*}_{T})^{-1}, \Id\right] + R_{D_T}\left[h, g \circ (g^{*}_{T})^{-1}\circ y_T\right] \\
\end{aligned}
\end{equation}
\end{small}
Since $\mathcal{H}_2$ is a universal Lipschitz hypothesis class,
\begin{small}
\begin{equation}\label{eq:prf1seq2}
\begin{aligned}
R_{D^y_T}\left[g\circ (g^{*}_{T})^{-1}, \Id\right] =& R_{D^y_T}\left[g\circ (g^{*}_{T})^{-1}, g\circ g^{-1}\right] \\
\lesssim& R_{D^y_T}\left[(g^{*}_{T})^{-1}, g^{-1}\right] \\
\end{aligned}
\end{equation}
\end{small}
By the definition of discrepancy and by the factor triangle inequality,
\begin{small}
\begin{equation}\label{eq:prf1seq3}
\begin{aligned}
R_{D^y_T}&\left[(g^{*}_{T})^{-1}, g^{-1}\right] = R_{a\circ D^{y}_T}\left[(b^{*}_{T})^{-1}, b^{-1}\right] \\
\lesssim& R_{D^y_S}\left[(b^{*}_{T})^{-1}, b^{-1}\right] + \disc_{\mathcal{H}^{-1}_4}(a\circ D^y_T, D^y_S)\\
=& R_{a^{-1}\circ D^y_S}\left[(b^{*}_{T})^{-1}\circ a, b^{-1}\circ a\right] \\
&+ \disc_{\mathcal{H}^{-1}_4}(a\circ D^y_T, D^y_S)\\
=& R_{a^{-1}\circ D^y_S}\left[(g^{*}_{T})^{-1}, g^{-1}\right] + \disc_{\mathcal{H}^{-1}_4}(a\circ D^y_T, D^y_S)\\
\end{aligned}
\end{equation}
\end{small}
Since $\mathcal{H}_2$ is a Bi-Lipschitz hypothesis class,
\begin{small}
\begin{equation}\label{eq:prf1seq4}
\begin{aligned}
R_{a^{-1}\circ D^y_S}&\left[(g^{*}_{T})^{-1}, g^{-1}\right]
\lesssim R_{a^{-1}\circ D^y_S}\left[g \circ (g^{*}_{T})^{-1}, g \circ g^{-1}\right]  \\
=& R_{a^{-1}\circ D^y_S}\left[g \circ (g^{*}_{T})^{-1}, \Id\right]  \\
=& R_{D_S}\left[g \circ (g^{*}_{T})^{-1} \circ a^{-1} \circ y_S, a^{-1}\circ y_S\right]  \\
\end{aligned}
\end{equation}
\end{small}
By the factor triangle inequality,
\begin{small}
\begin{equation}\label{eq:prf1seq5}
\begin{aligned}
R_{a^{-1}\circ D^y_S}&\left[(g^{*}_{T})^{-1}, g^{-1}\right] \\
\lesssim& R_{D_S}\left[g \circ (g^{*}_{T})^{-1} \circ a^{-1} \circ y_S, h\right] +R_{D_S}\left[h, a^{-1}\circ y_S\right]
\end{aligned}
\end{equation}
\end{small}
Since $\mathcal{H}_3$ and $\mathcal{H}_2$ are Bi-Lipschitz hypothesis classes, 
\begin{small}
\begin{equation}\label{eq:prf1seq6}
\begin{aligned}
R_{a^{-1}\circ D^y_S}\left[(g^{*}_{T})^{-1}, g^{-1}\right] \lesssim R_{D_S}\left[y_S, a \circ h^*_T\right] +R_{D_S}\left[a\circ h, y_S\right] \\
\end{aligned}
\end{equation}
\end{small}
By the factor triangle inequality,
\begin{small}
\begin{equation}\label{eq:prf1seq7}
\begin{aligned}
R_{a^{-1}\circ D^y_S}&\left[(g^{*}_{T})^{-1}, g^{-1}\right] \lesssim R_{D_S}\left[h^*_S, y_S\right] \\
&+R_{D_S}\left[a \circ h^*_T,h^*_S\right] +R_{D_S}\left[a\circ h, y_S\right] \\
\end{aligned}
\end{equation}
\end{small}
In addition, since $\mathcal{H}_2$ is a universal Bi-Lipschitz hypothesis class,
\begin{small}
\begin{equation}\label{eq:prf1seq8}
\begin{aligned}
R_{D_T}\left[h, g \circ (g^{*}_{T})^{-1}\circ y_T\right]
\lesssim& R_{D_T}\left[h^*_{T}, y_T \right]\\
\end{aligned}
\end{equation}
\end{small}
Combining Eqs.~\ref{eq:prf1seq1},~\ref{eq:prf1seq3},~\ref{eq:prf1seq7},~\ref{eq:prf1seq8} leads to the bound.
\end{proof}

Thm.~\ref{thm:main2}, which is illustrated in Fig.~\ref{fig:ODAfig2}, upper bounds the target generalization risk, $R_{D_T}[h,y_T]$. This bound is the sum between the source generalization risk between the adapted hypothesis $a\circ h$ and $y_S$, $R_{D_S}[a\circ h,y_S]$, the discrepancy between the adapted target output distribution $a \circ D^y_T$ and the source output distribution $D^y_S$, $\disc_{\mathcal{H}^{-1}_4}(a \circ D^y_T, D^y_S)$ and an unmeasurable constant, $R_{D_T}\left[h^{*}_{T}, y_T \right]+R_{D_S}\left[a\circ h^*_T,h^*_S \right]$. 

Note that $f$ can be removed, i.e., $\mathcal H_1 = \{\Id\}$. The advantage of using a non-trivial $f$ is that is that it allows the first component of the hypothesis from source to target to be non-invertible.

\begin{figure}[t]
\begin{center}
\begin{tikzpicture}[auto,scale=0.8]

    \node (in1) at ([shift={(-2.5,0)}] rep1)   
    {\begin{scriptsize}$D_S$\end{scriptsize}};
    \node (in2) at ([shift={(-2.5,-4.5)}] rep1)
    {\begin{scriptsize}$D_T$\end{scriptsize}};
    
    \node (rep1) 
    {\begin{scriptsize}$f\circ D_S$\end{scriptsize}};
    
    \node (mid1) at ([shift={(3,0)}] rep1) 
    {\begin{scriptsize}$b \circ f\circ D_S$\end{scriptsize}};
    
    \node (mid2) at ([shift={(6,0)}] rep1) 
    {\begin{scriptsize}$a \circ h\circ D_S$\end{scriptsize}};

    \node (mid3) at ([shift={(3,-3.5)}] rep1) 
    {\begin{scriptsize}$a \circ D^y_T$\end{scriptsize}};
    
    \node (out1) at ([shift={(3,-1.85)}] rep1)
    {\begin{scriptsize}$D^y_S$\end{scriptsize}};

    \node (out3) at ([shift={(6,-4.5)}] rep1) 
    {\begin{scriptsize}$D^y_T$\end{scriptsize}};
    
    \draw[->] (in1) -- (rep1) node[midway] 
    {\begin{scriptsize}$f$\end{scriptsize}};
    
    \draw[->] (rep1) -- (mid1) node[midway] 
    {\begin{scriptsize}$b$\end{scriptsize}};
    
    \path[-stealth] (out1) edge node [right] 
    {\begin{scriptsize}risk\end{scriptsize}} (mid1) ;
    \path[-stealth] (mid1) edge (out1) ;
    
    \path[-stealth] (mid1) edge node [pos=0.5,right,above] 
    {\begin{scriptsize}$=$\end{scriptsize}} (mid2) ;
    \path[-stealth] (mid2) edge (mid1) ;
    
    \path[-stealth] (out1) edge node [right] 
    {\begin{scriptsize}disc\end{scriptsize}} (mid3) ;
    \path[-stealth] (mid3) edge (out1) ;
   
    \draw[->] (in2) -- (out3) node[midway] 
    {\begin{scriptsize}$y_T$\end{scriptsize}};
    
    \path [l] (out3.north) -- ++(0,0.67)  -- node[pos=0.55,above] 
    {\begin{scriptsize}$a$\end{scriptsize}} (mid3.east) ;
    
    \path [l] (in1.south) -- ++(0,-1.5725)  -- node[pos=0.55] 
    {\begin{scriptsize}$y_S$\end{scriptsize}} (out1.west) ;

\end{tikzpicture}
\end{center}
\caption{\label{fig:ODAfig2} Analogy based output-side domain adaptation, where the learned functions are $h = g \circ f$ and $a$ where $g = a^{-1} \circ b$. The vertical edges stand for the discrepancy between the the two distributions $D^y_S$ and $a\circ D^y_T$ and the risk between $a\circ h$ and $y_S$ on the source distribution, $D_S$. In addition, $a \circ h = b \circ f$.}
\end{figure}
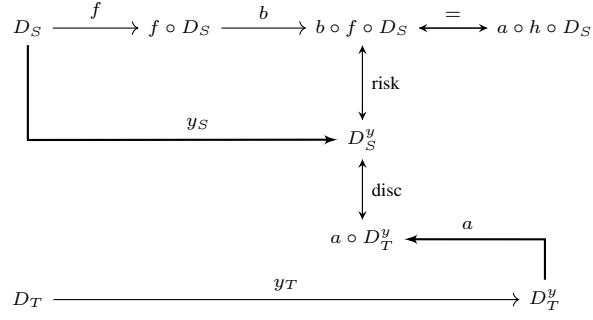

\section{Two-sided domain adaptation}

This setting is a special case of both unsupervised domain adaptation and unsupervised output-side domain adaptation. In this case, there are two domains, $(D_1,y_1)$ and $(D_2,y_2)$ and the learning algorithm is provided with four datasets: one includes i.i.d input instances from the first domain, the second includes labels of i.i.d instances from the first domain and the other two are the same for the second domain. The unlabeled input samples and output samples from each domain are not paired in any sense. Formally, we have two distributions, $D_1$ and $D_2$, and two target functions, $y_1$, $y_2$. The algorithm has access to the following four datasets, as illustrated in Fig.~\ref{fig:illustration}(c):
\begin{small}
\begin{equation}
\begin{aligned}
\{x_i\}^{m_1}_{i=1},\{y_1(x_j)\}^{m_2}_{j=1}  &\text{ such that } x_i,x_j \stackrel{\textnormal{i.i.d}}{\sim} D_1 \\
\{x_k\}^{n_1}_{k=1},\{y_2(x_l)\}^{n_2}_{l=1}  &\text{ such that } x_k,x_l \stackrel{\textnormal{i.i.d}}{\sim} D_2 
\end{aligned}
\end{equation}
\end{small}

We define a new type of discrepancy that measures if the relationships between two pairs of distributions $(D_{1,1},D_{1,2})$ and $(D_{2,1},D_{2,2})$ are similar. 
\begin{definition}[Quad discrepancy]\label{def:disc}
Let $\mathcal{C}$ be a set of functions from $A$ to $B$ and let $\ell: B\times B  \rightarrow \mathbb{R}_+$ be a loss function over $B$. The relationships discrepancy distance $\rdisc_{\mathcal{C}}$ between two pairs of distributions over $A$, $(D_{1,1},D_{1,2})$ and $(D_{2,1},D_{2,2})$,  is defined as follows,
\begin{small}
\begin{equation}
\begin{aligned}
&\rdisc_{\mathcal{C}}
  \begin{bmatrix}
    D_{1,1} & D_{1,2} \\
    D_{2,1} & D_{2,2}
  \end{bmatrix}
\\
&:= \sup_{c_1,c_2 \in \mathcal{C}}\Big\vert U_{D_{1,1},D_{1,2}}[c_1,c_2]-U_{D_{2,1},D_{2,2}}[c_1,c_2]\Big\vert
\end{aligned}
\end{equation}
\end{small}
Where, $U_{D_1,D_2}[c_1,c_2] := R_{D_1}[c_1,c_2] - R_{D_2}[c_1,c_2].$
\end{definition}
This new type of discrepancy measures the similarity between the relationships in two pairs of distributions. This quantity is at most, the sum of the discrepancies of each pair separately. Nevertheless, it might be a lot smaller. For example, we can take an arbitrary pair of distributions $(D_{1,1},D_{1,2}) = (D_{2,1},D_{2,2}) := (D_1,D_2)$ such that $\disc_{\mathcal{C}}(D_1,D_2)$ is (relatively) large and obtain 
\begin{small}
$$\rdisc_{\mathcal{C}}
  \begin{bmatrix}
    D_{1,1} & D_{1,2} \\
    D_{2,1} & D_{2,2}
  \end{bmatrix} = 0$$ \end{small} 
  while the sum of the discrepancies is (relatively) large. This follows since the relationships in the pair $(D_{1,1},D_{1,2})$ are the same relationships in the pair $(D_{2,1},D_{2,2})$. But, on the other hand, $D_{1}$ and $D_{2}$ not very much similar to each other. In addition, we consider that for any distribution $D$, 
\begin{small}
\begin{equation}
  \rdisc_{\mathcal{C}}
  \begin{bmatrix}
    D_{1} & D_{2} \\
    D & D
  \end{bmatrix} 
  = \rdisc_{\mathcal{C}}
  \begin{bmatrix}
    D_1 & D \\
    D_2 & D
  \end{bmatrix} = \disc_{\mathcal{C}}(D_1,D_2)
\end{equation}
\end{small}
We use this new type of discrepancy in order to bound the distance between discrepancies.

\begin{lemma}\label{lem:r-disc}
Let $(D_{1,1},D_{1,2})$ and $(D_{2,1},D_{2,2})$ be two pairs of distributions. Let $\mathcal{C}$ be any functions class. Then,
\begin{small}
\begin{equation}
\begin{aligned}
\Big\vert\disc_{\mathcal{C}}&(D_{1,1},D_{1,2}) - \disc_{\mathcal{C}}(D_{2,1},D_{2,2}) \Big\vert\\
&\leq \rdisc_{\mathcal{C}}
  \begin{bmatrix}
    D_{1,1} & D_{1,2} \\
    D_{2,1} & D_{2,2}
  \end{bmatrix}
\end{aligned}
\end{equation}
\end{small}
\end{lemma}

\begin{proof}
Let $c_1,c_2 \in \mathcal{C}$, we denote:
\begin{small}
\begin{equation}
\begin{aligned}
U_{1}[c_1,c_2] := \Big\vert U_{D_{1,1},D_{1,2}}[c_1,c_2]\Big\vert
\end{aligned}
\end{equation}
\end{small}
In addition, we denote, $U_{2}[c_1,c_2]$ analogously for the second pair. By the reversed triangle inequality,
\begin{small}
\begin{equation}\label{eq:lem1eq1}
\begin{aligned}
&\Big\vert U_{1}[c_1,c_2] - U_{2}[c_1,c_2]\Big\vert \leq \\
&\Big\vert U_{D_{1,1},D_{1,2}}[c_1,c_2] - U_{D_{2,1},D_{2,2}}[c_1,c_2]\Big\vert
\end{aligned}
\end{equation}
\end{small}
Therefore, 
\begin{small}
\begin{equation}\label{eq:lem1eq2}
\begin{aligned}
&U_{1}[c_1,c_2] \leq  U_{2}[c_1,c_2] \\
&+ \Big\vert U_{D_{1,1},D_{1,2}}[c_1,c_2] - U_{D_{2,1},D_{2,2}}[c_1,c_2]\Big\vert
\end{aligned}
\end{equation}
\end{small}
We consider that, for all $i=1,2$:
\begin{small}
\begin{equation}\label{eq:lem1eq3}
\begin{aligned}
\sup_{c_1,c_2\in \mathcal{C}}U_{i}[c_1,c_2] = \disc_{\mathcal{C}}(D_{i,1},D_{i,2}) 
\end{aligned}
\end{equation}
\end{small}
\begin{small}
\begin{equation}\label{eq:lem1eq4}
\begin{aligned}
&\sup_{c_1,c_2\in\mathcal{C}}\Big\vert U_{D_{1,1},D_{1,2}}[c_1,c_2] - U_{D_{2,1},D_{2,2}}[c_1,c_2]\Big\vert \\ &=\rdisc_{\mathcal{C}}
  \begin{bmatrix}
    D_{1,1} & D_{1,2} \\
    D_{2,1} & D_{2,2}
  \end{bmatrix}
\end{aligned}
\end{equation}
\end{small}
Finally, by taking $\sup_{c_1,c_2\in\mathcal{C}}$ in both sides of Eq.~\ref{eq:lem1eq2}, combined with Eqs.~\ref{eq:lem1eq3},~\ref{eq:lem1eq4} we obtain that:
\begin{small}
\begin{equation}\label{eq:lem1eq6}
\begin{aligned}
\disc_{\mathcal{C}}(D_{1,1},D_{1,2}) - \disc_{\mathcal{C}}(D_{2,1},D_{2,2}) \\
\leq \rdisc_{\mathcal{C}}
  \begin{bmatrix}
    D_{1,1} & D_{1,2} \\
    D_{2,1} & D_{2,2}
  \end{bmatrix}
\end{aligned}
\end{equation}
\end{small}
Eq.~\ref{eq:lem1eq6} is symmetric with respect to the two pairs of distributions and the desired bound is obtained.
\end{proof}
\begin{corollary} \label{cogan} For all $h_1,h_2 \in \mathcal{H}$ and $a_1,a_2:\mathcal Y \rightarrow \mathbb{R}^
k$,
\begin{small}
\begin{equation}
\label{eq:cogan}
\begin{aligned}
\Big\vert \disc_{\mathcal{C}}&(a_1 \circ h_1 \circ D_{1},a_1 \circ D^y_{1})-\disc_{\mathcal{C}}(a_2 \circ h_2 \circ D_{2},a_2 \circ D^y_{2}) \Big\vert\\
&\leq \rdisc_{\mathcal{C}}
  \begin{bmatrix}
    a_1 \circ h_1 \circ D_{1} & a_1 \circ D^y_{1} \\
    a_2 \circ h_2 \circ D_{2} & a_2 \circ D^y_{2}
  \end{bmatrix}
\end{aligned}
\end{equation}
\end{small}
\end{corollary}

\proof{A special case of Lem.~\ref{lem:r-disc} with $D_{i,1} = a_i \circ h_i \circ D_{i}$, $D_{i,2} = a_i \circ D^y_{i}$ for $i=1,2$.}

Corollary~\ref{cogan} (illustrated in Fig.~\ref{fig:coGANfig}) motivates the training of two GANs together, as is done in CoGAN~\cite{cogan}. Specifically, we have already pointed out in Sec.~\ref{sec:pregan} that taking the supremum over two functions $c_1$ and $c_2$ is analogous to finding the best discriminator. The right hand side of Eq.~\ref{eq:cogan} can, therefore, be interpreted as finding the best discriminator $d$ that separates the learned function from the target function of the source domain much better than the analog functions of the target domain, or vice versa. If both domains are equally inseparable, e.g., by making sure that $d$ fails to discriminate in both, then the r.h.s is small. 

In CoGAN, $D_1=D_2$ is a distribution over random vectors, and two generative functions $h_1$ and $h_2$ are learned to create fake samples from two output distributions $D_1^y$ and $D_2^y$ ($y_1$ differs from $y_2$, and so the two output domains differ). The two learned functions share common layers, which correspond to $h_1 = g_1 \circ f$ and similarly for $h_2$, with a shared $f$. The discriminators between the the real and the fake samples in both domains $d_1 = d \circ a_1$ and $d_2 = d \circ a_2$, respectively, share most of their layers ($d$). This, as mentioned above, is analog to applying the same pair of functions $c_1$ and $c_2$ to the two domains in the quad discrepancy. 

\begin{small}
\begin{figure}[t]
\begin{center}
\begin{tikzpicture}[auto,scale=1]

    \node (in1) {\begin{scriptsize}$D_1$\end{scriptsize}};
    \node (in2) at ([shift={(0,-1.5)}] in1) 
    {\begin{scriptsize}$D_2$\end{scriptsize}};
    
    \node (out1) at ([shift={(2,0)}] in1) {\begin{scriptsize}$h_1\circ D_1$\end{scriptsize}};
    \node (out2) at ([shift={(2,-1.5)}] in1) {\begin{scriptsize}$h_2\circ D_2$\end{scriptsize}};

    \node (mid1) at ([shift={(4.5,0)}] in1) {\begin{scriptsize}$a_1\circ h_1\circ D_1$\end{scriptsize}};
    \node (mid2) at ([shift={(4.5,-1.5)}] in1) {\begin{scriptsize}$a_2\circ h_2\circ D_2$\end{scriptsize}};

    \node (ext1) at ([shift={(4.5,1)}] in1) {\begin{scriptsize}$D^y_1$\end{scriptsize}};
    \node (ext2) at ([shift={(4.5,-2.5)}] in1) {\begin{scriptsize}$D^y_2$\end{scriptsize}};

      \node (extt1) at ([shift={(7,0)}] in1) {\begin{scriptsize}$a_1 \circ D^y_1$\end{scriptsize}};
    \node (extt2) at ([shift={(7,-1.5)}] in1) {\begin{scriptsize}$a_2\circ D^y_2$\end{scriptsize}};

    \draw[->] (in1)  -- (out1) node[midway] {\begin{scriptsize}$h_1$\end{scriptsize}};
    \draw[->] (in2) -- (out2) node[midway] {\begin{scriptsize}$h_2$\end{scriptsize}};

    \draw[->] (out1) -- (mid1) node[midway] {\begin{scriptsize}$a_1$\end{scriptsize}};
    \draw[->] (out2) -- (mid2) node[midway] {\begin{scriptsize}$a_2$\end{scriptsize}};

	\path [l] (in1.north) -- ++(0,0.77)  -- node[pos=0.55] 
    {\begin{scriptsize}$y_1$\end{scriptsize}} 
    (ext1.west) ;    
    
    \path [l] (in2.south) -- ++(0,-0.77)  -- node[pos=0.55] 
    {\begin{scriptsize}$y_2$\end{scriptsize}} 
    (ext2.west) ;

    \path [l] (ext1.east) --node[pos=0.55] 
    {\begin{scriptsize}$a_1$\end{scriptsize}} ++(2.175,0)  --  
    (extt1.north) ;
    
    \path [l] (ext2.east) --node[pos=0.55] 
    {\begin{scriptsize}$a_2$\end{scriptsize}}  ++(2.175,0)  -- 
    (extt2.south) ;
    
    \path[-stealth] (mid1) edge node [pos=0.5,right,above] 
    {\begin{scriptsize}\end{scriptsize}} (extt2) ;
    \path[-stealth] (extt2) edge (mid1) ;
    
    \path[-stealth] (mid2) edge node [pos=0.8,right,below] 
    {\begin{scriptsize}$\;\;\;\;\;\;\;\rdisc$\end{scriptsize}} (extt1) ;
    \path[-stealth] (extt1) edge (mid2) ;

\end{tikzpicture}
\end{center}
\caption{\label{fig:coGANfig} Two-sided domain adaptation. The learned functions are $h_1$ and $h_2$, and the two auxiliary functions $a_1$ and $a_2$. The horizontal edges denote for the functions between the distributions. $a_1$ and $a_2$ stand for the adapters for each pair. The crossing two-sided arrows stand for the quad discrepancy between the pairs of distributions $(a_1\circ h_1 \circ D_1,a_1\circ D^y_1)$ and $(a_2\circ h_2 \circ D_2,a_2\circ D^y_2)$.}
\end{figure}
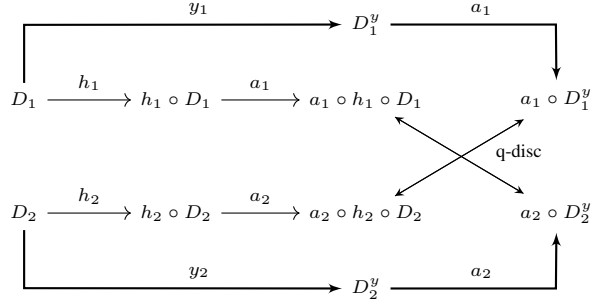
\end{small}

\begin{small}
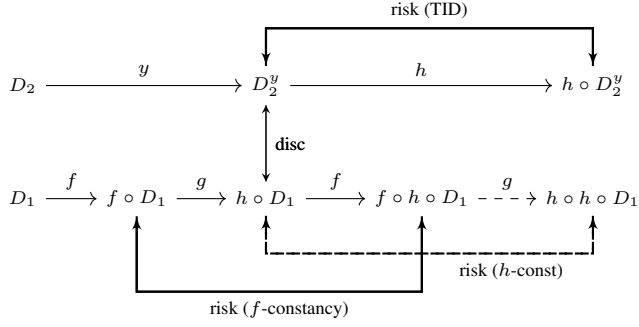
\begin{figure}[t]
\begin{center}
\begin{tikzpicture}[auto,scale=1]

	\node (in1) {\begin{scriptsize}$D_2$\end{scriptsize}};
    \node (out11) at ([shift={(3.2,0)}] in1) {\begin{scriptsize}$D^y_2$\end{scriptsize}};
    \node (out12) at ([shift={(7.5,0)}] in1) {\begin{scriptsize}$h \circ D^y_2$\end{scriptsize}};
    
    \node (in2) at ([shift={(0,-1.5)}] in1) {\begin{scriptsize}$D_1$\end{scriptsize}};
    \node (rep1) at ([shift={(1.5,-1.5)}] in1) {\begin{scriptsize}$f \circ D_1$\end{scriptsize}};
    \node (out21) at ([shift={(3.2,-1.5)}] in1) {\begin{scriptsize}$h \circ D_1$\end{scriptsize}};
    \node (rep2) at ([shift={(5.25,-1.5)}] in1) {\begin{scriptsize}$f \circ h \circ D_1$\end{scriptsize}};
    \node (out22) at ([shift={(7.5,-1.5)}] in1) {\begin{scriptsize}$h \circ h \circ D_1$\end{scriptsize}};
    
    \path[-stealth] (out11) edge node [right] {\begin{scriptsize}$\disc$\end{scriptsize}} (out21) ;
    \path[-stealth] (out21) edge node [right] {\begin{scriptsize}$\disc$\end{scriptsize}} (out11) ;
    
    \draw[->] (in1)  -- (out11) node[midway] {\begin{scriptsize}$y$\end{scriptsize}};
    \draw[->] (out11) -- (out12) node[midway] {\begin{scriptsize}$h$\end{scriptsize}};

    \draw[->] (in2) -- (rep1) node[midway] {\begin{scriptsize}$f$\end{scriptsize}};
    \draw[->] (rep1) -- (out21) node[midway] {\begin{scriptsize}$g$\end{scriptsize}};
    
    \draw[dashed,->] (rep2) -- (out22) node[midway] {\begin{scriptsize}$g$\end{scriptsize}};
    \draw[->] (out21) -- (rep2) node[midway] {\begin{scriptsize}$f$\end{scriptsize}};
    
    \path [l] (out11.north) -- ++(0,0.5)  -- node[pos=0.5] {\begin{scriptsize}risk (TID)\end{scriptsize}} ++(4.3,0) --  (out12.north);
    
    \path [l] (out12.north) -- ++(0,0.5)  -- ++(-4.3,0) --  (out11.north);

    \path [l] (rep1.south) -- ++(0,-1)  -- node[pos=0.5,below] {\begin{scriptsize}risk ($f$-constancy)\end{scriptsize}} ++(3.75,0) --  (rep2.south);
    
    \path [l] (rep2.south) -- ++(0,-1)  -- ++(-3.75,0) --  (rep1.south);
    
    \path [dashed,l] (out21.south) -- ++(0,-0.5)  -- node[pos=0.75,below] {\begin{scriptsize}risk ($h$-const)\end{scriptsize}} ++(4.3,0) --  (out22.south);
    
    \path [dashed,l] (out22.south) -- ++(0,-0.5)  -- ++(-4.3,0) --  (out21.south);
\end{tikzpicture}
\end{center}
\caption{\label{fig:DTfig} Domain Transfer. The learned function is $h = g \circ f$. The horizontal two-sided edges denote the TID and $f$-constancy risks that are used by the algorithm. The vertical two-sided edge stands for the discrepancy between $D^y_2$ and $h\circ D_1$. The dashed edges stand for the $h$-constancy risk that is required only in Thm.~\ref{thm:main3}, but is not necessary in Cor.~\ref{cor:dtn}.}
\end{figure}
\end{small}

\section{Domain Transfer}

In the cross domain transfer problem, the task is to learn a generative function that transfers samples from the input domain $\cal X$ to the output domain domain $\cal Y$. It was recently presented in~\cite{DBLP:journals/corr/TaigmanPW16}, where a GAN based solution was able to convincingly transform face images into caricatures from a specific domain. In comparison to the superficially related problem of style transfer~\cite{Gatys_2016_CVPR}, the cross domain problem was shown to be more semantic, in the sense that it adheres to the structure of the output domain.

The learning algorithm is provided with only two unlabeled datasets: one includes i.i.d samples from the input distribution and the second includes i.i.d samples from the output distribution. Formally, we have two distributions, $D_1$ and $D_2$, and a target function, $y$. The algorithm has access to the following two datasets,
\begin{small}
\begin{equation}
\begin{aligned}
\{x_i\}^{m}_{i=1} &\text{ such that } x_i \stackrel{\textnormal{i.i.d}}{\sim} D_1\\
\{y(x_j)\}^n_{j=1} &\text{ such that } x_j \stackrel{\textnormal{i.i.d}}{\sim} D_2 
\end{aligned}
\end{equation}
\end{small}
This is illustrated in Fig.~\ref{fig:illustration}(d). The goal is to fit a function $h=g\circ f\in \mathcal{H}$ that is closest to,
\begin{small}
\begin{equation}
\label{eq:32}
\begin{aligned}
\inf_{h\in \mathcal{H}} R_{D_1}[h,y]
\end{aligned}
\end{equation}
\end{small}

It is assumed that: (i) $f$ is a fixed pre-trained feature map and, therefore, $\mathcal{H} = \left\{g \circ f \big\vert g \in \mathcal{H}_2\right\}$; and (ii) $y$ is idempotent, i.e, $y \circ y \equiv y$. 

For example, in~\cite{DBLP:journals/corr/TaigmanPW16}, $f$ is the DeepFace representations function~\cite{taigman2014deepface} and the function $y$ maps face images to emoji caricatures. In addition, applying $y$ on an emoji gives the same emoji. 

Note that according to the terminology of~\cite{DBLP:journals/corr/TaigmanPW16}, $D_1$ and $D_2$ are the source and target distributions respectively. However, this conflicts with the terminology of domain adaptation, since the loss in Eq.~\ref{eq:32} is measured over $D_1$. In domain adaptation, loss is measured over the target distribution. 

We denote $D^y_2 := y \circ D_2$. The following bound is illustrated in Fig.~\ref{fig:DTfig}.
\begin{theorem}[Domain transfer bound]\label{thm:main3}
If Assumption~\ref{assmp:triangle} holds, then for all $h=g\circ f \in \mathcal{H}$,
\begin{small}
\begin{equation}
\begin{aligned}
R_{D_1}[h,y] \lesssim &
R_{D^y_2}[h,\Id] + R_{D_1}[h\circ h,h]\\
&+R_{D_1}[f \circ h,f]
+ \disc_{\mathcal{H}}(D^y_2, h \circ D_1)+\lambda 
\end{aligned}
\end{equation}
\end{small}
Here, $\lambda = \min_{h\in \mathcal{H}} \left\{R_{D^y_2}[h,\Id] + R_{D_1}[h,y]\right\}$ and $h^*=g^*\circ f$ is the corresponding minimizer. We also assume that $g^*$ is Lipschitz with respect to $\ell$.
\end{theorem}

\begin{proof} By the factor triangle inequality, 
\begin{small}
\begin{equation}
\begin{aligned}
R_{D_1}&[h,y] 
\lesssim R_{D_1}[h\circ h,h]+R_{D_1}[h\circ h,y]\\
\lesssim& R_{D_1}[h\circ h,h]+R_{D_1}[h\circ h,h^* \circ h] \\ 
&+ R_{D_1}[h^* \circ h,y]\\
=& R_{D_1}[h\circ h,h]+R_{h\circ D_1}[h,h^*] + R_{D_1}[h^* \circ h,y]\\
\end{aligned}
\end{equation}
\end{small}
By the definition of discrepancy,
\begin{small}
\begin{equation}
\begin{aligned}
R_{D_1}[h,y]\lesssim& R_{D_1}[h\circ h,h]+R_{D^y_2}[h,h^*] \\ 
&+ R_{D_1}[h^* \circ h,y] + \disc_{\mathcal{H}}(D^y_2, h \circ D_1)\\
=& R_{D_1}[h\circ h,h]+R_{D^y_2}[h,h^*] \\ 
&+ R_{D_1}[g^* \circ f \circ h,y] + \disc_{\mathcal{H}}(D^y_2, h \circ D_1)\\
\end{aligned}
\end{equation}
\end{small}
By the factor triangle inequality,
\begin{small}
\begin{equation}
\begin{aligned}
R_{D_1}&[h,y]\lesssim R_{D_1}[h\circ h,h]+R_{D^y_2}[h,h^*] \\ 
&+ R_{D_1}[g^* \circ f \circ h,g^* \circ f]+ R_{D_1}[g^* \circ f,y] \\
&+ \disc_{\mathcal{H}}(D^y_2, h \circ D_1)\\
\end{aligned}
\end{equation}
\end{small}
Since $\ell(g^*(a),g^*(b)) \leq L \cdot \ell(a,b)$, we have,
\begin{small}
\begin{equation}
\begin{aligned}
R_{D_1}&[h,y]\lesssim R_{D_1}[h\circ h,h]+R_{D^y_2}[h,h^*] \\
&+ R_{D_1}[f \circ h,f]+ R_{D_1}[h^*,y] \\
&+ \disc_{\mathcal{H}}(D^y_2, h \circ D_1)\\ 
\end{aligned}
\end{equation}
\end{small}
Again, by the factor triangle inequality,
\begin{small}
\begin{equation}
\begin{aligned}
R_{D_1}&[h,y]\lesssim R_{D_1}[h\circ h,h]+R_{D^y_2}[h,\Id] \\ 
&+ R_{D_1}[f \circ h,f]+ R_{D^y_2}[h^*,\Id]  \\
&+ R_{D_1}[h^*,y]+ \disc_{\mathcal{H}}(D^y_2, h \circ D_1)\qedhere
\end{aligned}
\end{equation}
\end{small}
\end{proof}

\begin{corollary}\label{cor:dtn} In the setting of Thm.~\ref{thm:main3}. If $\mathcal{H}_2$ is a universal Lipschitz hypothesis class, then for all $h=g\circ f \in \mathcal{H}$,
\begin{small}
\begin{equation}
\begin{aligned}
R_{D_1}[h,y] \lesssim &
R_{D^y_2}[h,\Id] +R_{D_1}[f \circ h,f]\\
&+ \disc_{\mathcal{H}}(D^y_2, h \circ D_1)+\lambda 
\end{aligned}
\end{equation}
\end{small}
\end{corollary}

\begin{proof}
Since $\mathcal{H}_2$ is a universal Lipschitz hypothesis class,
\begin{small}
\begin{equation}
\begin{aligned}
R_{D_1}[h\circ h,h] &= R_{D_1}[g \circ f \circ h, g \circ f] \\
&\lesssim  R_{D_1}[f \circ h, f]
\end{aligned}
\end{equation}
\end{small}
Therefore, by Thm.~\ref{thm:main3} we obtain the desired bound.
\end{proof}

The last corollary matches the method of~\cite{DBLP:journals/corr/TaigmanPW16}. The first term $R_{D^y_2}[h,\Id]$ is the $L_{\text{TID}}$ part of their loss, which, for the emoji generation application, states that emoji caricatures are mapped to themselves. The second term $R_{D_1}[f \circ h,f]$ corresponds to their $L_{\text{CONST}}$ term, which states that the DeepFace representations of the input face image and the resulting caricature are similar. In our analysis this constancy is not assumed as part of the problem formulation, instead it stems from the idempotency of $y$.

The third term $\disc_{\mathcal{H}}(D^y_2, h \circ D_1)$ is the algorithm's GAN element that compares generated caricatures to the training dataset of the unlabeled emoji. Note that in~\cite{DBLP:journals/corr/TaigmanPW16}, a ternary GAN is used, which also involves the distribution of generated images where the input is from $D_2$. However, no clear advantage to the ternary GAN over the binary GAN is observed. Lastly the $\lambda$ factor captures the complexity of the hypothesis class $\cal H$, which depends on the chosen architecture of the neural networks that instantiate $g$.  

\section{Conclusion}
Problems involving domain shift receive an increasing amount of attention, as the field of machine learning moves its focus away from the vanilla supervised learning scenarios to new combinations of supervised, unsupervised and transfer learning. 

We analyze several new unsupervised and semi-supervised paradigms. While the ODA problem is, as far as we know,  completely novel, two other problems we define provide theoretical foundations to recent algorithms.

\bibliographystyle{plain}
\bibliography{ODA}
\end{document}